\documentclass[runningheads]{llncs}
\usepackage[T1]{fontenc}
\usepackage{csquotes}
\usepackage{amsmath}
\usepackage{amssymb}
\usepackage{graphicx}
\usepackage{todonotes}
\setuptodonotes{color=green!40}
\usepackage{hyperref}
\usepackage{orcidlink}
\usepackage{color}

\definecolor{Colors-A}{RGB}{230,159,0}  % orange
\definecolor{Colors-C}{RGB}{86,180,233}  % sky-blue
\definecolor{Colors-B}{RGB}{0,158,115}  % bluish green
\definecolor{Colors-D}{RGB}{213,94,0}  % vermilion (red)
\definecolor{Colors-E}{RGB}{204,121,167}  % reddish purple
\definecolor{Colors-F}{RGB}{0,114,178}  % blue
\definecolor{Colors-G}{RGB}{240,228,66}  % yellow 
\definecolor{Input-Color}{RGB}{229,158,197}
\definecolor{Param-Color}{RGB}{239,231,119}
\definecolor{Output-Color}{RGB}{141,176,252}
\newcommand{\Reals}{\mathbb{R}}
\newcommand{\RealsPos}{\mathbb{R}_{> 0}}
\newcommand{\RealsNonNeg}{\mathbb{R}_{\geq 0}}
\newcommand{\Nats}{\mathbb{N}}

\newcommand{\varvec}[1]{\vec{#1}}
\newcommand{\mat}[1]{\vec{#1}}

\newcommand{\degin}{\mathrm{degin}}

\DeclareMathOperator*{\argmax}{arg\,max}

\newcommand{\ReLU}[1]{{\left[#1\right]}^+}
\newcommand{\NNExtra}[1][]{%
  \mathrm{net}_{#1}
}
\newcommand{\NN}{\NNExtra[\vec{\theta}]}
\newcommand{\NNIn}[1][{}]{\NNExtra[\mathrm{In}_{#1}]}
\newcommand{\NNSat}[1][{}]{\NNExtra[\mathrm{Sat}_{#1}]}

\newcommand{\textcite}[1]{\textbf{#1}}

\begin{document}
\title{Verifying Global Neural Network Specifications using Hyperproperties}
\titlerunning{Verifying Global Neural Network Specifications using Hyperproperties}
% If the paper title is too long for the running head, you can set
% an abbreviated paper title here
%
\author{%
  David Boetius\orcidlink{0000-0002-9071-1695} \and
  Stefan Leue\orcidlink{0000-0002-4259-624X}
}
\authorrunning{D. Boetius \and S. Leue}
% First names are abbreviated in the running head.
% If there are more than two authors, 'et al.' is used.
%
\institute{University of Konstanz, 78457 Konstanz, Germany
\email{\{david.boetius,stefan.leue\}@uni-konstanz.de}}
\maketitle
\begin{abstract}
  Current approaches to neural network verification focus on 
  specifications that target small regions 
  around known input data points, such as local robustness.
  Thus, using these approaches, we can not obtain guarantees for inputs 
  that are not close to known inputs.
  Yet, it is highly likely that a neural network will encounter such
  truly unseen inputs during its application.
  We study global specifications that~---~when satisfied~---~provide
  guarantees for all potential inputs.
  We introduce a hyperproperty formalism that allows for 
  expressing global specifications such as monotonicity,
  Lipschitz continuity, global robustness, and dependency fairness.
  Our formalism enables verifying global specifications using existing neural network
  verification approaches by leveraging capabilities for verifying general computational
  graphs.
  Thereby, we extend the scope of guarantees that can be provided 
  using existing methods.
  Recent success in verifying specific global specifications shows
  that attaining strong guarantees for all potential data points
  is feasible.

\end{abstract}
\section{Introduction}
Deep learning is a game changer for research, education, business 
and beyond~\cite{DwivediKshetriHughesEtAl2023,EloundouManningMishkinEtAl2023}.
Yet, we remain unable to provide strong guarantees on the behaviour of
neural networks.
In particular, while neural network verification in principle can provide
strong guarantees, current approaches almost exclusively consider \emph{local} 
specifications~\cite{%
  BakLiuJohnson2021,%
  FerrariMuellerJovanovicEtAl2022,%
  HenriksenLomuscio2021,%
  MuellerBrixBakEtAl2022,%
  SinghGehrPueschelEtAl2019,%
  ZhangWangXuEtAl2022%
}
that only apply to small regions around known input data points.
This means that the currently widely-used specifications 
only sparsely cover the input space, 
providing no guarantees for inputs that are not close to known inputs. 
In contrast, \emph{global} specifications cover the entire input space.

We propose a specification formalism for neural networks
that encompasses a rich class of global specifications while enabling 
verification using existing verifier technology.
In particular, we show how monotonicity, Lipschitz continuity,
two notions of global robustness~\cite{%
  KatzBarrettDillEtAl2017a,LeinoWangFredrikson2021%
}, 
and dependency fairness~\cite{%
  GalhotraBrunMeliou2017,UrbanChristakisWuestholzEtAl2020%
} can be expressed using our formalism.

As noted in~\cite{SeshiaDesaiDreossiEtAl2018}, global specifications
such as monotonicity and global robustness
are hyperproperties~\cite{ClarksonSchneider2008}.
In difference to regular properties that only consider one network execution at a time, 
hyperproperties relate executions for several inputs of the same neural network to each other.
This allows us, for example, to express a na{\"\i}ve notion of global robustness
stating that an arbitrary input and a second input that lies close 
need to receive the same classification.

A central aspect of our formalism is that we use auxiliary neural networks
to define input sets and output sets.
% These auxiliary neural networks serve the purpose of 
% making the complex input and output sets of 
% global specifications accessible for existing 
% neural network verification approaches. 
% When coupled with \emph{self-composition}~\cite{ClarksonSchneider2008}, the auxiliary networks enable 
% verification approaches with capabilities for verifying general computational
% graphs~\cite{XuShiZhangEtAl2020b} to verify hyperproperties.
By leveraging capabilities for verifying general computational
graphs~\cite{XuShiZhangEtAl2020b}, the auxiliary networks,
together with \emph{self-composition}~\cite{ClarksonSchneider2008}, 
allow for verifying hyperproperties using existing neural network
verification approaches.
Here, the role of the auxiliary neural networks is to make complex
hyperproperty input and output sets accessible
to existing verification approaches. 
Concretely, we design an auxiliary neural network to generate
the tuples of inputs that need to be compared
to determine whether a hyperproperty is satisfied. 
Another auxiliary neural network detects whether the outputs
a network produces for these inputs satisfy the output constraint.
For the na{\"\i}ve notion of global robustness, this means that 
we derive a neural network that generates arbitrary pairs of 
inputs that are close to each other
and another neural network that detects whether two outputs 
represent the same classification.
Importantly, these auxiliary neural networks \emph{exactly} capture the targeted input
and output constraints using standard neural network components.

Recent success in verifying global robustness~\cite{WangHuangZhu2022}
and global individual fairness~\cite{UrbanChristakisWuestholzEtAl2020}
demonstrates that verifying global specifications is feasible.
Our formalism is a general framework for global specifications targeting
existing verifiers~\cite{%
  FerrariMuellerJovanovicEtAl2022,%
  KatzHuangIbelingEtAl2019,%
  SinghGehrPueschelEtAl2019,%
  ZhangWangXuEtAl2022%
}.
While our formalism does not alleviate the need for specialised techniques,
such as the Interleaving Twin Encoding~\cite{WangHuangZhu2022},
it allows for
% \textbf{1.} comparing existing general-purpose verifiers
% with specialised approaches for global specifications and 
% \textbf{2.} applying general-purpose verifiers to 
% global specifications for which no specialised verifiers exist.
\begin{enumerate}
  \item Comparing general-purpose verifiers
    with specialised verifiers for specific global specifications and 
  \item Applying general-purpose verifiers to 
  global specifications for which no specialised verifier exists.
\end{enumerate}

\section{Preliminaries}
We consider verifying whether a neural network conforms to a 
global specification.
Neural networks are computational graphs~\cite{GoodfellowBengioCourville2016}.
Global specifications are formalised using 
hyperproperties~\cite{ClarksonSchneider2008,SeshiaDesaiDreossiEtAl2018}.
\begin{definition}[Computational Graph]
  A \emph{computational graph} is a directed acyclic graph with computations~\((V, E, h)\), 
  where~\(V = \{1, \ldots, L\}\) with~\(L \in \Nats\) 
  is the set of nodes,~\(E \subseteq V \times V\)
  is the edge relation and~\(h = (h_1, \ldots, h_L)\) is the computations tuple.
  Let~\(\degin: V \to \Nats\) denote the in-degree.
  The computation of node~\(i \in V\) is~\(%
  h_i: \Reals^{m_{k_1}} \times \cdots \times \Reals^{m_{k_{\degin(i)}}} \to \Reals^{m_i},
  \) where~\(m_i \in \Nats\) is the output dimension of node~\(i\)
  and~\(k_1, \ldots, k_{\degin(i)} \in \{i \mid (k, i) \in E\}\)
  with~\(k_1 \leq \cdots \leq k_{\degin(i)}\) are the direct predecessors of~\(i\).
\end{definition}

\begin{definition}[Neural Network]
  A \emph{neural network}~\(\NN: \Reals^n \to \Reals^m\), \(n, m \in \Nats\)
  is a composition of affine transformations and non-affine activation functions 
  represented by a computational graph~\((V, E, h)\) 
  with a source~\(i\) and a single sink~\(j\),
  such that~\(h_i: \{\emptyset\} \to \Reals^{n}\)
  and~\(h_j: \Reals^{m_{k_1}} \times \cdots \times \Reals^{m_{k_{\degin(j)}}} \to \Reals^m\).
  The source~\(i\) is the \emph{input} of~\(\NN\).
  The remaining sources of the computational graph together form the 
  \emph{parameters}~\(\vec{\theta}\) of~\(\NN\).
  The sink~\(j\) is the \emph{output} of~\(\NN\).
  For classification tasks,~\(\argmax_{j = 1}^m \NN(\vec{x})\)
  is the class~\(\NN\) assigns to an input~\(\vec{x} \in \Reals^n\).
\end{definition}
Figure~\ref{fig:cg-residual-unit} contains the computational graph
of a residual unit~\cite{HeZhangRenEtAl2016a} as an example.
This graph defines the steps necessary for computing the output
of a residual unit, given an input.
It also allows for computing gradients and verifying a residual unit.
Assume we want to compute the outputs of a neural network for an input~\(\vec{x} \in \Reals^n\).
Also, assume we have a parameter assignment~\(\vec{\theta}\).
We assign~\(\vec{x}\) to the network input node~\(i\) and 
the corresponding parameter values to the remaining sources.
Now, computing the outputs corresponds to a forward walk over the computational graph,
propagating the computation results of each node to its direct successors.
Similarly, a backwards walk from sinks to sources allows for computing the gradients
of the sink with respect to each source (backpropagation).
Forward and backwards walks also allow for computing certified lower and upper bounds
on the network output that can be used for verifying the 
neural network~\cite{XuShiZhangEtAl2020b}.

\begin{figure}
  \centering
  \includegraphics[width=.75\linewidth]{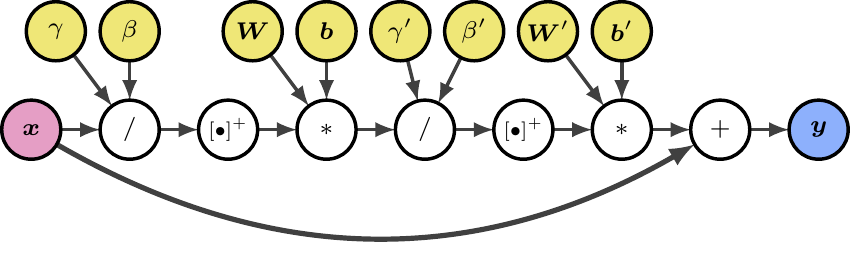}
  \caption{%
    \textbf{The computational graph of a residual unit~\cite{HeZhangRenEtAl2016a}.}
    In this figure,~\(\ast\) denotes convolution,~\(/\) denotes batch 
    normalisation,~\(\ReLU{\bullet}\) denotes ReLU, and~\(+\) denotes addition.
    We use pink nodes~\includegraphics{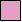} for inputs, 
    yellow~\includegraphics{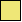} for parameters,
    and blue~\includegraphics{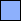} for outputs.
  }\label{fig:cg-residual-unit}
\end{figure}

Verifying a neural network means that we want to automatically 
prove or disprove whether the neural network satisfies a \emph{specification}.
A specification is a set of properties.

\begin{definition}[Property]\label{defn:prop}
  A \emph{property}~\(\varphi = (\mathcal{X}_\varphi, \mathcal{Y}_\varphi)\) is a tuple
  of an \emph{input set}~\(\mathcal{X}_\varphi \subseteq \Reals^n\) and an
  \emph{output set}~\(\mathcal{Y}_\varphi \subseteq \Reals^m\),~\(n, m \in \Nats\).
  We write~\(\NN \vDash \varphi\) when a neural network~\(\NN: \Reals^n \to \Reals^m\)
  \emph{satisfies} the property~\(\varphi\). Specifically,
  \begin{align*}
    \NN \vDash \varphi &\Leftrightarrow 
    \forall \vec{x} \in \mathcal{X}_\varphi: \NN(\vec{x}) \in \mathcal{Y}_\varphi.
  \end{align*}
  We call an input~\(\vec{x} \in \mathcal{X}_\varphi\) for which~\(\NN(\vec{x}) \notin \mathcal{Y}_\varphi\)
  a \emph{counterexample}.
\end{definition}

A verifier determines whether a neural network~\(\NN\) satisfies a property~\(\varphi\).
We require verifiers to \textbf{1.} report property satisfaction if and only if
the property is indeed satisfied (\emph{soundness}) and 
\textbf{2.} to terminate (\emph{completeness}).
In this paper, we only require verifiers
to support bounded hyperrectangles as property input sets and 
the non-negative real numbers as output set.
Practically, verifiers can also handle more complicated input and output sets.

For formalising global specifications, we make use of \emph{hyperproperties}.
Hyperproperties extend properties by considering multiple input variables
and input-dependent output sets.
\begin{definition}[Hyperproperty]\label{defn:hyperprop}
  A \emph{hyperproperty}~\(\psi=(\mathcal{X}_\psi, \mathcal{Y}_\psi)\) is a tuple
  of a \emph{multi-variable input set}~\(\mathcal{X}_\psi \subseteq {\left(\Reals^n\right)}^v\)
  and an \emph{input-dependent output set}
  \begin{align*}
    \mathcal{Y}_\psi &\subseteq 
    \underbrace{\Reals^n \times \cdots \times \Reals^n}_{v\ \text{times}} \times 
    \underbrace{\Reals^m \times \cdots \times \Reals^m}_{v\ \text{times}},
  \end{align*}
  where~\(n, m, v \in \Nats\).
  For a neural network~\(\NN: \Reals^n \to \Reals^m\),
  we write~\(\NN \vDash \psi\) if
  \begin{align*}
    \forall \vec{x}^{(1)}, \ldots, \vec{x}^{(v)} \in \mathcal{X}_\psi:
    \left(
      \vec{x}^{(1)}, \ldots, \vec{x}^{(v)}, 
      \NN{\left(\vec{x}^{(1)}\right)}, \ldots,
      \NN{\left(\vec{x}^{(v)}\right)}
    \right) \in \mathcal{Y}_\psi.
  \end{align*}
\end{definition}

\section{Formalising Global Specifications using Hyperproperties}
Global specifications allow for expressing desired behaviour for the entire input
domain of a neural network while local specifications only apply to small regions 
around known inputs.
This property of local specifications brings with it 
that we have a fixed reference point for each property
in a local specification.
We typically do not have such a fixed reference point for global specifications,
since they apply to the entire input domain.
% This makes it necessary to compare multiple executions of the same network, 
% in turn requiring the use of hyperproperties for formalising global specifications.

For example, a local robustness property expresses that
a classifier assigns the same class to all inputs 
that lie within a small~\(L_p\)-ball~\(\mathcal{B}_p(\vec{x})\) 
around a fixed input point~\(\vec{x}\).
Because we have this fixed input~\(\vec{x}\) as a reference point, 
we know the class that should be assigned to all the inputs in~\(\mathcal{B}_p(\vec{x})\).
Knowing this class allows for judging whether an input~\(\vec{x}' \in \mathcal{B}_p(\vec{x})\) 
is a counterexample to the local robustness property by executing the network
once for~\(\vec{x}'\).

If we now look at global robustness, we find that it does not suffice to consider
a single execution of a network to check for specification violations.
As the inputs now are arbitrary inputs from the entire input domain,
we can not determine whether robustness is violated by looking only at the output
for one input~\(\vec{x}^{(1)}\).
Instead, we need to find another 
input~\(\vec{x}^{(2)} \in \mathcal{B}_p{\left(\vec{x}^{(1)}\right)}\)
such that the classes that a network assigns to~\(\vec{x}^{(1)}\) 
and~\(\vec{x}^{(2)}\) do not match.
Only in pair, these inputs form a counterexample.
The necessity to compare outputs for multiple inputs requires us to adopt
hyperproperties for formalising global specifications.

If we look more closely at our example of global robustness, we find that
requiring the points in all~\(L_p\)-balls to have the same output forces
the network to produce the same output for all inputs.
This means that we also have to consider more complicated output sets for global specifications.
In this case, we either need to allow small changes in class scores (Example~\ref{example:global-robust-1})
or devise special rules for points close to the decision boundary (Example~\ref{example:global-robust-2}).
Furthermore, if we express global robustness as Lipschitz 
continuity~\cite{CisseBojanowskiGraveEtAl2017} (Example~\ref{example:lipschitz-cont}), 
our output set needs to be \emph{input-dependent}.
This means that it does not suffice to only compare network outputs with network outputs 
to determine whether a specification is violated. 
Instead, we also need to take the inputs that lead to the observed outputs into account.

For the reasons outlined above, we consider hyperproperties with multi-variable input sets
and input-dependent output sets as in Definition~\ref{defn:hyperprop} for formalising
global specifications.
To leverage existing neural network verification approaches for verifying
these hyperproperties, we express the multi-variable input set and the 
input-dependent output set using auxiliary neural networks.

\begin{definition}[Neural-Network-Defined Hyperproperty]\label{defn:nn-hyperprop}
  Let~\(n, m, v, w \in \Nats\).
  A \emph{Neural-Network-Defined Hyperproperty (NNDH)} is a 
  hyperproperty~\(\psi = \left(\mathcal{X}_\psi, \mathcal{Y}_\psi\right)\),
  where~\(\mathcal{X}_\psi = \left\{\NNIn(\vec{w}) \mid \vec{w} \in \mathcal{W}\right\}\)
  and
  \begin{equation*}
    \mathcal{Y}_\psi = \left\{
      \vec{x}^{(1)}, \ldots, \vec{x}^{(v)}, \vec{y}^{(1)}, \ldots, \vec{y}^{(v)} \,\left|\,
      \NNSat{\left(\vec{x}^{(1)}, \ldots, \vec{x}^{(v)}, \vec{y}^{(1)}, \ldots, \vec{y}^{(v)}\right)} \geq 0 
    \right.\right\},
  \end{equation*}
  where~\(\mathcal{W} \subset \Reals^w\) is a bounded hyperrectangle
  and~\(\NNIn: \Reals^w \to {(\Reals^n)}^v\)
  and~\(\NNSat: \underbrace{\Reals^n \times \cdots \times \Reals^n}_{v\ \text{times}} \times 
  \underbrace{\Reals^m \times \cdots \times \Reals^m}_{v\ \text{times}} \to \Reals\) 
  are neural networks.
\end{definition}
We can think of the neural network~\(\NNIn\) as generating the multi-variable input set
from a single-variable hyperrectangular input space.
The neural network~\(\NNSat\) serves as a 
\emph{satisfaction function}~\cite{BauerMarquartBoetiusLeueEtAl2021}
for the output set.
A satisfaction function is non-negative if and only if an output~---~or, in this case,
a tuple of inputs and outputs~---~lies within the output set or a property or hyperproperty.

It is central to Definition~\ref{defn:nn-hyperprop} that~\(\NNIn\)
and~\(\NNSat\) do not \emph{approximate} our desired input and output set,
but express them \emph{exactly}. 
Usually, we train neural networks to approximate a potentially unknown 
relationship between inputs and outputs.
The neural networks~\(\NNIn\) and~\(\NNSat\), however, are not trained but
carefully constructed to generate our desired input and output set.
As such, these auxiliary neural networks are relatively simple structures
in this paper.
Their main purpose is to make hyperproperties accessible for existing neural
network verification approaches.

We now provide several concrete examples of NNDHs including concrete~\(\NNIn\)
and~\(\NNSat\) networks.
We formalise global monotonicity, 
two notions of global robustness~\cite{KatzBarrettDillEtAl2017a,LeinoWangFredrikson2021}, 
Lipschitz continuity, and dependency 
fairness~\cite{GalhotraBrunMeliou2017,UrbanChristakisWuestholzEtAl2020} as NNDHs.
Afterwards, we show how NNDHs can be verified using existing neural network verifiers
that can handle general computational graphs.

In the following, let~\(\mathcal{X} \subset \Reals^n\) be the bounded hyperrectangular input
domain of the neural network under consideration.
This domain is determined by the target application. 
In the case of image classification, for example,~\(\mathcal{X}\) would be the (normalised)
pixel space.

\begin{example}[Global Monotonicity]\label{example:monot}
  Monotonicity is a desired behaviour of a neural network in
  applications from medicine to aviation~\cite{SivaramanFarnadiMillsteinEtAl2020}.
  Here, we formalise that the output~\(j \in \{1, \ldots, m\}\) may not \emph{increase}
  when input~\(i \in \{1, \ldots, n\}\) increases.
  Non-decreasing monotonicity can be formalised analogously.
  We formalise global monotonicity as a hyperproperty~\(%
    \psi_{M} = \left(\mathcal{X}_{\psi_M}, \mathcal{Y}_{\psi_M}\right)
  \), 
  where the input set~\(\mathcal{X}_{\psi_M} \subseteq \mathcal{X} \times \mathcal{X}\)
  and output set~\(%
    \mathcal{Y}_{\psi_M} \subset \Reals^n \times \Reals^n \times \Reals^m \times \Reals^m
  \) are
  \begin{align*}
  \mathcal{X}_{\psi_M} &= \left\{ 
      \vec{x}^{(1)}, \vec{x}^{(2)} 
      \;\left|\;  \vec{x}^{(2)}_i \geq \vec{x}^{(1)}_i
    \right.\right\} \\
    \mathcal{Y}_{\psi_M} &= \left\{ 
      \vec{x}^{(1)}, \vec{x}^{(2)}, \vec{y}^{(1)}, \vec{y}^{(2)}
      \;\left|\; \vec{y}^{(2)}_j \leq \vec{y}^{(1)}_j
    \right.\right\}.
  \end{align*}
  To generate these sets using neural networks to obtain an NNDH, we define
  \begin{equation*}
      \mathcal{W}_M 
    = \left\{
      \vec{x}^{(1)}_1, \ldots, \vec{x}^{(1)}_n, \vec{x}^{(2)}_1, \ldots, \vec{x}^{(2)}_n 
      \;\left|\; \vec{x}^{(1)}, \vec{x}^{(2)} \in \mathcal{X}
    \right.\right\},
  \end{equation*}
  \begin{gather*}
    \NNIn[M]{\left(\vec{x}^{(1)}_1, \ldots, \vec{x}^{(1)}_n, 
       \vec{x}^{(2)}_1, \ldots, \vec{x}^{(2)}_n\right)} 
    = \left(\vec{x}'^{(1)}, \vec{x}'^{(2)}\right), \\
      \!\begin{aligned}
        \text{where}\ \vec{x}'^{(1)} 
        &= \left( 
          \vec{x}^{(1)}_1, \ldots, 
          \min{\left(\vec{x}^{(1)}_i, \vec{x}^{(2)}_i\right)}, 
          \ldots, \vec{x}^{(1)}_n
        \right) \\
        \vec{x}'^{(2)} 
        &= \left( 
          \vec{x}^{(2)}_1, \ldots, 
          \max{\left(\vec{x}^{(1)}_i, \vec{x}^{(2)}_i\right)}, 
          \ldots, \vec{x}^{(2)}_n
        \right),
      \end{aligned}
  \end{gather*}
  and
  \begin{equation*}
    \NNSat[M]{\left(\vec{x}^{(1)}, \vec{x}^{(2)}, \vec{y}^{(1)}, \vec{y}^{(2)}\right)}
    = \vec{y}^{(1)}_j - \vec{y}^{(2)}_j.
  \end{equation*}
  % \begin{align*}
  %   \NNIn[M]{\left(\vec{x}^{(1)}_1, \ldots, \vec{x}^{(1)}_n, 
  %      \vec{x}^{(2)}_1, \ldots, \vec{x}^{(2)}_n\right)} 
  %   &= \begin{array}[t]{rl}
  %       \Big(&\vec{x}^{(1)}_1, \ldots, 
  %        \min{\left(\vec{x}^{(1)}_i, \vec{x}^{(2)}_i\right)}, \ldots, \vec{x}^{(1)}_n, \\
  %     & \vec{x}^{(2)}_1, \ldots, 
  %        \max{\left(\vec{x}^{(1)}_i, \vec{x}^{(2)}_i\right)}, \ldots, \vec{x}^{(2)}_n
  %   \Big)
  %   \end{array} \\
  %   \NNSat[M]{\left(\vec{x}^{(1)}, \vec{x}^{(2)}, \vec{y}^{(1)}, \vec{y}^{(2)}\right)}
  %   &= \vec{y}^{(1)}_j - \vec{y}^{(2)}_j.
  % \end{align*}
  The function~\(\NNSat[M]\) is a neural network with a single affine layer.
  Concerning~\(\NNIn[M]\), we can compute~\(\max\) either using a maxpooling layer
  or by leveraging~\(%
    \forall a, b \in \Reals: \max(a, b) = \ReLU{a - b} + b
  \)
  where~\(\ReLU{\bullet} = \max(\bullet, 0)\) is ReLU.
  Furthermore, since~\(\forall a, b \in \Reals: \min(a, b) = -\max(-a, -b)\), 
  we can also compute~\(\min\) in a neural network.
  Therefore,~\(\mathcal{W}_M\),~\(\NNIn[M]\) and~\(\NNSat[M]\) together form an NNDH 
  having~\(\mathcal{X}_{\psi_M}\) as its input set and~\(\mathcal{Y}_{\psi_M}\)
  as its output set.
\end{example}

\begin{example}[Global~\(L_\infty\) Robustness following~\cite{KatzBarrettDillEtAl2017a}]%
  \label{example:global-robust-1}
  Neural networks are susceptible to adversarial attacks where slightly modifying the input
  allows an attacker to control the output produced by a 
  neural network~\cite{SzegedyZarembaSutskeverEtAl2014}.
  This is a safety concern, for example, for traffic sign 
  recognition~\cite{EykholtEvtimovFernandesEtAl2018}
  and biometric authentication using face recognition~\cite{SharifBhagavatulaBauerEtAl2016}.
  In this example, we express~\(L_\infty\) global robustness
  according to~\cite{KatzBarrettDillEtAl2017a} as an NNDH.\@
  This specification limits how much the output of a neural network may change
  for inputs that lie within an~\(L_\infty\)-ball of a certain size.
  Let~\(\delta, \varepsilon \in \RealsPos\) be the radius of the \(L_\infty\)-ball
  and the permitted magnitude of change, respectively.
  Let
  \begin{align*}
    \mathcal{W}_{R} 
    &= \left\{%
      \vec{x}_1, \ldots, \vec{x}_n, \varvec{\tau}_1, \ldots, \varvec{\tau}_n
      \mid \vec{x} \in \mathcal{X}, \varvec{\tau} \in {[-\delta, \delta]}^n
    \right\} \\
    \NNIn[R]{\left( \vec{x}_1, \ldots, \vec{x}_n, \varvec{\tau}_1, \ldots, \varvec{\tau}_n \right)}
    &= \left(\vec{x}, \mathrm{project}_{\mathcal{X}}(\vec{x} + \varvec{\tau})\right) \\
    \NNSat[R1]{\left( \vec{x}^{(1)}, \vec{x}^{(2)}, \vec{y}^{(1)}, \vec{y}^{(2)}\right)}
    &= \varepsilon - {\left\| \vec{y}^{(1)} - \vec{y}^{(2)} \right\|}_\infty
    = \varepsilon - \max_{j=1}^m \left| \vec{y}^{(1)}_j - \vec{y}^{(2)}_j \right|,
  \end{align*}
  where~\(\mathrm{project}_{\mathcal{X}}\) computes the projection into the 
  hyperrectangle~\(\mathcal{X}\).
  Projecting a point~\(\vec{x}\) into a hyperrectangle corresponds to
  computing the minimum between each coordinate and the lower boundary of the
  hyperrectangle and the maximum between each coordinate and the upper boundary
  of the hyperrectangle.
  As we show in Example~\ref{example:monot}, we can compute minima and maxima in a
  neural network.
  Similarly,~\(\NNSat{R1}\) computes a maximum and absolute values, which we can compute
  by leveraging~\(\forall a \in \Reals: |a| = \max(a, -a)\).
  Overall,~\(\mathcal{W}_R\),~\(\NNIn[R]\), and~\(\NNSat[R1]\) define an 
  NNDH~\(\psi_{R1} = (\mathcal{X}_{\psi_{R}}, \mathcal{Y}_{\psi_{R1}})\),
  where~\(\mathcal{X}_{\psi_{R}} \subset \mathcal{X} \times \mathcal{X}\)
  and~\(\mathcal{Y}_{\psi_{R1}} \subset \Reals^n \times \Reals^n \times \Reals^m \times \Reals^m\),
  with
  \begin{align*}
    \mathcal{X}_{\psi_{R}} &= \left\{ 
      \vec{x}^{(1)}, \vec{x}^{(2)} 
      \;\left|\; {\left\| \vec{x}^{(1)} - \vec{x}^{(2)} \right\|}_\infty \leq \delta
    \right.\right\} \\
    \mathcal{Y}_{\psi_{R1}} &= \left\{
      \vec{x}^{(1)}, \vec{x}^{(2)}, \vec{y}^{(1)}, \vec{y}^{(2)} 
      \;\left|\; {\left\| \vec{y}^{(1)} - \vec{y}^{(2)} \right\|}_\infty \leq \varepsilon
    \right.\right\}.
  \end{align*}
  This captures that a network is globally robust as defined in~\cite{KatzBarrettDillEtAl2017a}.
\end{example}

\begin{example}[Global~\(L_\infty\) Robustness following~\cite{LeinoWangFredrikson2021}]%
  \label{example:global-robust-2}
  We also present an alternative definition of global robustness using an extra class
  representing non-robustness at an input point~\cite{LeinoWangFredrikson2021}.
  This definition may be more desirable in some applications, as it still permits 
  non-robustness for noise-only \emph{rubbish class} inputs~\cite{GoodfellowShlensSzegedy2015} 
  that lie off the data manifold. 
  Let~\(\delta \in \RealsPos\) be as in Example~\ref{example:global-robust-1}.
  Assume the classifier network we are studying produces an additional output~\(\bot = m+1\)
  that shall indicate non-robustness.
  We reuse~\(\mathcal{X}_{\psi_R}\) from Example~\ref{example:global-robust-1}
  and define~\(\psi_{R2} = (\mathcal{X}_{\psi_R}, \mathcal{Y}_{\psi_{R2}})\),
  where~\(\mathcal{Y}_{\psi_{R2}} \subset \Reals^n \times \Reals^n \times \Reals^{m+1} \times \Reals^{m+1}\)
  and, concretely,
  \begin{equation*}
  \mathcal{Y}_{\psi_{R2}} = \left\{ 
        \vec{x}^{(1)}, \vec{x}^{(2)}, \vec{y}^{(1)}, \vec{y}^{(1)} 
        \;\left|\; 
        N\!R{\left(\vec{y}^{(1)}\right)} \vee
        N\!R{\left(\vec{y}^{(2)}\right)} \vee
        Same{\left(\vec{y}^{(1)}, \vec{y}^{(2)}\right)}
    \right.\right\},
  \end{equation*}
  where
  \begin{align*}
    N\!R(\vec{y}) &= \bigwedge_{j=1}^{m} \vec{y}_\bot^{(k)} \geq \vec{y}_j^{(k)} \\
    Same{\left(\vec{y}^{(1)}, \vec{y}^{(2)}\right)}
    &= 
      \bigvee_{j_1 = 1}^{m} \bigwedge_{k=1}^2 \bigwedge_{j_2=1}^m 
      \vec{y}_{j_1}^{(k)} \geq \vec{y}_{j_2}^{(k)}.
  \end{align*}
  Intuitively,~\(N\!R\) captures that the extra class~\(\bot\) is assigned
  to an input, while~\(Same\) captures that the 
  same class is assigned to~\(\vec{y}^{(1)}\) and~\(\vec{y}^{(2)}\)\footnote{%
    Strictly speaking,~\(Same\) only requires that there is an intersection between
    the largest elements of~\(\vec{y}^{(1)}\) and~\(\vec{y}^{(2)}\).
    This comes into play when the assigned class is ambiguous 
    due to an output having several largest elements.
  }.
  To construct a neural network~\(\NNSat[R2]\) that serves as a satisfaction function for~\(\psi_{R2}\),
  we note that for an arbitrary vector~\(\vec{u} \in \Reals^u\),~\(u \in \Nats\)
  \begin{align}
    & \bigvee_{a \in \mathcal{A}} \bigwedge_{b \in B(a)} 
    \vec{u}_{k_1(a,b)} \geq \vec{u}_{k_2(a,b)} \label{eqn:dnf}\\
    \Leftrightarrow \;\; &
    \left(
      \max_{a \in \mathcal{A}} \min_{b \in B(a)}
      \vec{u}_{k_1(a,b)} - \vec{u}_{k_2(a,b)}
    \right) \geq 0,
  \end{align}
  where~\(\mathcal{A}\) and~\(\mathcal{B}\) are finite 
  sets,~\(B: \mathcal{A} \to 2^{\mathcal{B}}\), 
  and~\(k_1, k_2: \mathcal{A} \times \mathcal{B} \to \Nats\).
  As we can transform any formula in propositional logic into Disjunctive Normal Form,
  we can bring the formula defining~\(\mathcal{Y}_{\psi_{R2}}\) into the form
  of Equation~\eqref{eqn:dnf}.
  Therefore, since we can compute~\(\min\) and~\(\max\) using a neural network (Example~\ref{example:monot}),
  we can define a neural network~\(\NNSat[R2]\) serving as a satisfaction function for~\(\psi_{R2}\).
  Together with~\(\mathcal{W}\) and~\(\NNIn[R]\) from Example~\ref{example:global-robust-1},~\(\NNSat[R2]\)
  defines an NNDH with the same input and output set as~\(\psi_{R2}\).
\end{example}

\begin{example}[Lipschitz Continuity]\label{example:lipschitz-cont}
  The Lipschitz continuity of a neural network is linked not only to
  robustness~\cite{SzegedyZarembaSutskeverEtAl2014}
  but also to fairness~\cite{DworkHardtPitassiEtAl2012}, 
  generalisation~\cite{BartlettFosterTelgarsky2017}, and 
  explainability~\cite{FelVigourouxCadeneEtAl2022}.
  While many neural network architectures are always Lipschitz continuous~\cite{%
    CisseBojanowskiGraveEtAl2017,% residual
    SzegedyZarembaSutskeverEtAl2014,% ReLU and Maxpool feed-forward
    RuanHuangKwiatkowska2018% feed-forward with other activation functions
  }, it is the magnitude of the Lipschitz constant that 
  matters~\cite{CisseBojanowskiGraveEtAl2017}.
  Let~\(K \in \RealsNonNeg\) be the desired global Lipschitz constant.
  Define~\(%
    \mathcal{W}_{C} = \left\{\left.
      \vec{x}^{(1)}_1, \ldots, \vec{x}^{(1)}_n, \vec{x}^{(2)}_1, \ldots, \vec{x}^{(2)}_n
      \;\right|\; \vec{x}^{(1)}, \vec{x}^{(2)} \in \mathcal{X}
    \right\}
  \) and
  \begin{align*}
    \NNIn[C]{\left(\vec{x}^{(1)}_1, \ldots, \vec{x}^{(1)}_n, \vec{x}^{(2)}_1, \ldots, \vec{x}^{(2)}_n\right)}
    &=
    \left(\vec{x}^{(1)}, \vec{x}^{(2)}\right) \\
    \NNSat[C]{\left(\vec{x}^{(1)}, \vec{x}^{(2)}, \vec{y}^{(1)}, \vec{y}^{(2)}\right)} 
    &=
    K{\left\|\vec{x}^{(1)} - \vec{x}^{(2)}\right\|}_\infty
    - {\left\|\vec{y}^{(1)} - \vec{y}^{(2)}\right\|}_\infty.
  \end{align*}
  First,~\(\NNIn[C]\) is an identity function and, thus, a trivial neural network.
  Then, by computing~\(\|\bullet\|_\infty\) as in Example~\ref{example:global-robust-1}
  in a neural network, we obtain an 
  NNDH~\(\psi_{C} = (\mathcal{X}_{\psi_C}, \mathcal{Y}_{\psi_C})\) with
  \begin{align*}
      \mathcal{X}_{\psi_C} &= \mathcal{X} \times \mathcal{X} \\
      \mathcal{Y}_{\psi_C} &= \left\{
        \vec{x}^{(1)}, \vec{x}^{(2)}, \vec{y}^{(1)}, \vec{y}^{(2)}
        \;\left|\;
        \left\| \vec{y}^{(1)} - \vec{y}^{(2)} \right\|_\infty \leq
        K\left\| \vec{x}^{(1)} - \vec{x}^{(2)} \right\|_\infty
      \right.\right\},
  \end{align*}
  which corresponds to Lipschitz continuity with Lipschitz constant~\(K\).
\end{example}

\begin{example}[Dependency Fairness]\label{example:dependency-fairness} 
  Machine learning applications from automated hiring~\cite{BogenRieke2018} to 
  image classification~\cite{PrabhuBirhane2020} bear the danger of producing 
  unfair machine-learning models. 
  However, in some applications, ensuring fairness may be legally
  required~\cite{PedreschiRuggieriTurini2008}. 
  One fairness requirement that we may pose is that 
  \enquote{similar individuals are treated similarly}~\cite{DworkHardtPitassiEtAl2012}. 
  \emph{Dependency fairness}~\cite{%
    GalhotraBrunMeliou2017,%
    UrbanChristakisWuestholzEtAl2020%
  } is a fairness criterion based on this idea\footnote{%
    We believe dependency fairness is an overly simplistic fairness criterion 
    as it can be trivially satisfied by withholding sensitive attributes 
    from the neural network, which is known to be insufficient 
    for real-world fairness~\cite{BarocasHardtNarayanan2019}.
    However, we still think that dependency fairness is suitable as an example
    for experimenting with verifying global specifications.
  }.
  Assume the first dimension of the input space is a categorical sensitive attribute
  with~\(A \in \Nats\) disjoint values.
  We consider two inputs to be \emph{similar}
  if they are equal except for the sensitive attribute.
  Dependency fairness specifies that all similar inputs
  are assigned to the same class.
  Let~\(\psi_{F} = (\mathcal{X}_{\psi_F}, \mathcal{Y}_{\psi_F})\),
  with~\(\mathcal{X}_{\psi_F} \subset \mathcal{X}^A\),~\(%
    \mathcal{Y}_{\psi_F} \subset {\left(\Reals^n\right)}^A \times {\left(\Reals^m\right)}^A
  \), where
  \begin{align*}
    \mathcal{X}_{\psi_F} &= \left\{
      \vec{x}^{(1)}, \ldots, \vec{x}^{(A)} 
      \;\left|\;
      \begin{array}{ll}
        \forall k \in \{1, \ldots, A\}: \\
          \quad \Big( \vec{x}^{(k)}_1 = k \wedge
          \forall i \in \{2, \ldots, n\}: 
          \vec{x}^{(1)}_i = \vec{x}^{(k)}_i
        \Big)
      \end{array}
    \right.\right\} \\
    \mathcal{Y}_{\psi_F} &= \left\{
      \vec{x}^{(1)}, \ldots, \vec{x}^{(A)}, \vec{y}^{(1)}, \ldots, \vec{y}^{(A)} \;\left|\;
      \bigvee_{j_1=1}^{m}
        \bigwedge_{k=1}^A
        \bigwedge_{j_2=1}^m 
        \vec{y}^{(k)}_{j_1} \geq \vec{y}^{(k)}_{j_2}
    \right.\right\}.
  \end{align*}
  We can construct a neural network satisfaction function~\(\NNSat[F]\) 
  for this property analogously to Example~\ref{example:global-robust-2}.
  The input set~\(\mathcal{X}_{\psi_F}\) consists of tuples of similar inputs
  which contain each value of the sensitive attribute in a fixed order.
  Let~\(\mat{A} \in \Reals^{n \times n}\) be the diagonal matrix
  with \(0, 1, \ldots, 1\) on its diagonal.
  Let~\(\mathrm{assign}: \Nats \times \Reals \to \Reals\) 
  be an affine function with~\(\mathrm{assign}(k, \vec{x}) = \vec{A}\vec{x} + {(k, 0, \ldots, 0)}^T\).
  Define~\(\mathcal{W} = \mathcal{X}\) and~\(
    \NNIn[F](\vec{x}) = \left( \mathrm{assign}(1, \vec{x}), \ldots, \mathrm{assign}(A, \vec{x}) \right)
  \).
  Since~\(\mathrm{assign}\) is affine,~\(\NNIn[F]\) is a neural network.
  Overall,~\(\mathcal{W}\),~\(\NNIn[F]\), and~\(\NNSat[F]\) define an NNDH
  with the same input and output set as~\(\psi_F\).
\end{example}

These examples demonstrate that Definition~\ref{defn:nn-hyperprop} is an expressive
specification formalism, despite restricting input and output sets to be defined by
neural networks.
It remains to show that we can indeed verify NNDHs using existing 
neural network verification approaches.
This builds upon the ability to verify general computational graphs.
In~\cite{XuShiZhangEtAl2020b}, the Linear Relaxation-based
Perturbation Analysis (LiRPA) framework is extended to general computational graphs.
LiRPA underlies verifiers such as 
\(\alpha,\!\beta\)-CROWN~\cite{ZhangWangXuEtAl2022}
and ERAN~\cite{SinghGehrPueschelEtAl2019},
and is used in Marabou~\cite{KatzHuangIbelingEtAl2019}
and MN-BaB~\cite{FerrariMuellerJovanovicEtAl2022}, among others.\@
Among these verifiers,~\(\alpha,\!\beta\)-CROWN already supports verifying
general computational graphs.

The central idea in verifying an NNDH~\(\psi\) is to compose the 
network to verify~\(\NN\) with itself and 
the networks~\(\NNIn[\psi]\) and~\(\NNSat[\psi]\)
that define the input and output set of~\(\psi\).
\begin{theorem}[NNDH Verification]\label{thm:nn-hyperprop-verif}
  Let~\(\psi = (\mathcal{X}_\psi, \mathcal{Y}_\psi)\) 
  with~\(\mathcal{W} \subseteq \Reals^w\),~\(\NNIn: \Reals^w \to {\left(\Reals^n\right)}^v\)
  and~\(\NNSat: {\left(\Reals^m\right)}^v \to \Reals\) be an NNDH.\@
  Let~\(\NN: \Reals^n \to \Reals^m\) be a neural network.
  Define~\(\NN': \Reals^w \to \Reals\) as
  \begin{align*}
    \NN'{(\vec{w})} &= \NNSat{\left(
        \vec{x}^{(1)}, \ldots, \vec{x}^{(v)},
        \NN{\left(\vec{x}^{(1)}\right)}, \ldots, 
        \NN{\left(\vec{x}^{(v)}\right)}
    \right)} \\
                    & \quad\; \text{where}\ \vec{x}^{(1)}, \ldots, \vec{x}^{(v)} = \NNIn(\vec{w}).
  \end{align*}
  Further, let~\(\varphi = (\mathcal{W}, \RealsNonNeg)\).
  It holds that~\(\NN' \vDash \varphi \Leftrightarrow \NN \vDash \psi\).
\end{theorem}
\begin{proof}
  Theorem~\ref{thm:nn-hyperprop-verif} follows from
  applying Definitions~\ref{defn:prop} and~\ref{defn:nn-hyperprop}.
  % Let~\(\psi\),~\(\mathcal{W}\),~\(\NNIn\)~\(\NNSat\),~\(\NN\),~\(\NN'\)
  % and~\(\varphi\) be as in Theorem~\ref{thm:nn-hyperprop-verif}.
  % \begin{align*}
  %     & \NN' \vDash \varphi \\
  %   = &\; \forall \vec{w} \in \mathcal{W}: \NN'(\vec{w}) \in \RealsNonNeg \\
  %   = &\; \forall \vec{w} \in \mathcal{W}: \NNSat{\left(
  %       \vec{x}^{(1)}, \ldots, \vec{x}^{(v)},
  %       \NN{\left(\vec{x}^{(1)}\right)}, \ldots, 
  %       \NN{\left(\vec{x}^{(v)}\right)}
  %   \right)} \geq 0 \\
  %     & \quad\; \text{where}\ \vec{x}^{(1)}, \ldots, \vec{x}^{(v)} = \NNIn(\vec{w}) \\
  %   = &\; \forall \vec{x}^{(1)}, \ldots, \vec{x}^{(v)} \in \mathcal{X}_\psi: \NNSat{\left(
  %       \vec{x}^{(1)}, \ldots, \vec{x}^{(v)},
  %       \NN{\left(\vec{x}^{(1)}\right)}, \ldots, 
  %       \NN{\left(\vec{x}^{(v)}\right)}
  %   \right)} \geq 0 \\
  %   = &\; \forall \vec{x}^{(1)}, \ldots, \vec{x}^{(v)} \in \mathcal{X}_\psi: 
  %       \left(\vec{x}^{(1)}, \ldots, \vec{x}^{(v)},
  %       \NN{\left(\vec{x}^{(1)}\right)}, \ldots, 
  %       \NN{\left(\vec{x}^{(v)}\right)}\right) \in \mathcal{Y}_\psi
  % \end{align*}
  \qed{}
\end{proof}

Figure~\ref{fig:self-composition} visualises~\(\NN'\) 
from Theorem~\ref{thm:nn-hyperprop-verif}.
We construct a new computational graph by generating several inputs
using~\(\NNIn\) and feeding each input to a separate copy of~\(\NN\).
Finally,~\(\NNSat\) takes the generated inputs and the output
of each copy of~\(\NN\) and computes the satisfaction function value.
Considering several copies of the same artefact is known as 
\emph{self-composition}~\cite{ClarksonSchneider2008}.
As Theorem~\ref{thm:nn-hyperprop-verif} shows, verifying an NNDH~\(\psi\)
corresponds to verifying a property~\(\varphi\) of the new computational graph~\(\NN'\).
Overall,~\(\NN'\) has a more complicated graph structure
than~\(\NN\), but it only contains computations that also appear
in~\(\NN\),~\(\NNIn\) or~\(\NNSat\).
Therefore,~\(\psi\) can be verified using verifiers that can verify~\(\NN\),~\(\NNIn\)
and~\(\NNSat\) and support general computational graphs.

\begin{figure}[t]
  \centering
  \includegraphics[width=\textwidth]{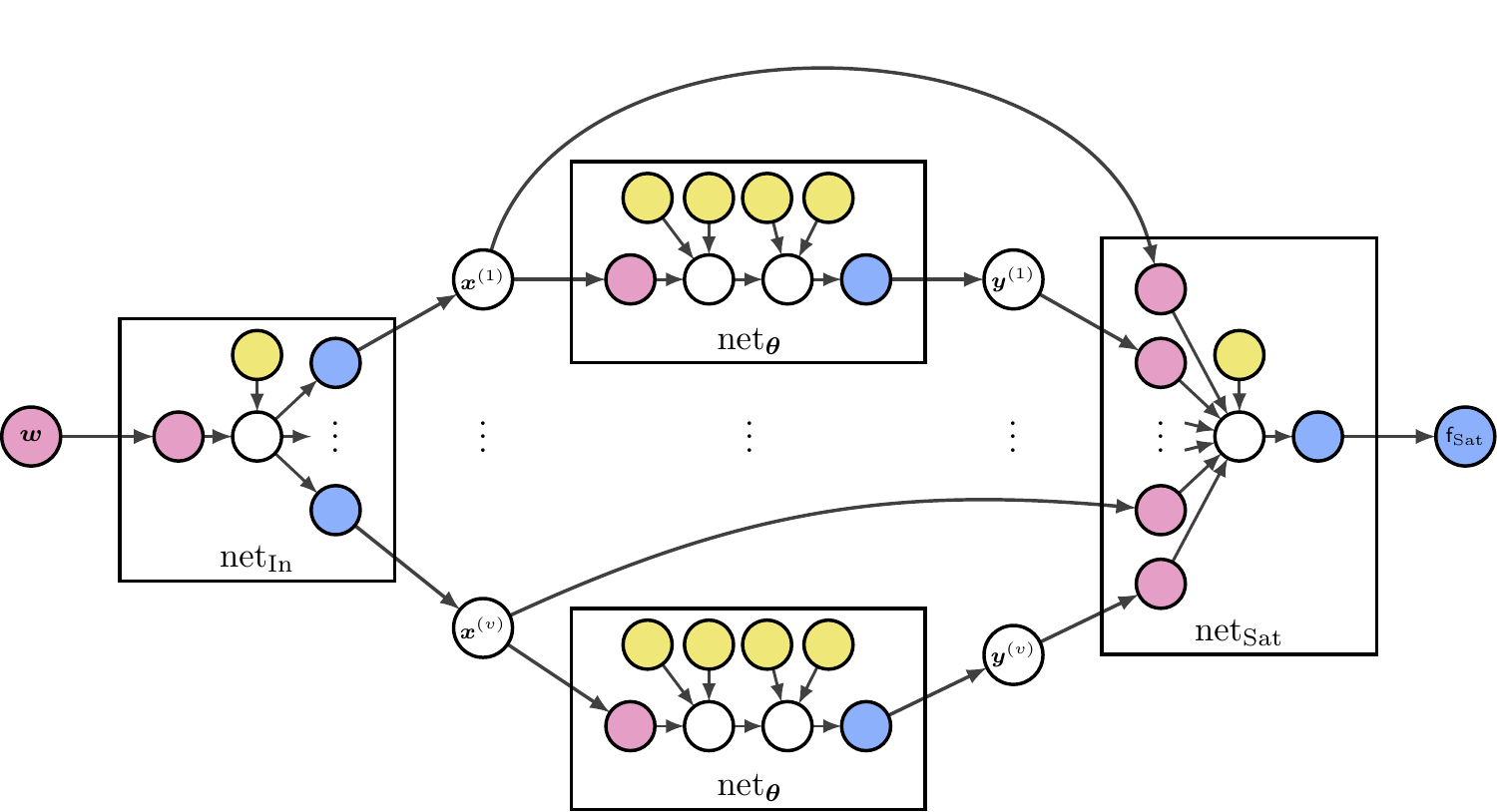}
  \caption{%
    \textbf{Computational Graph for Verifying NNDHs.}
    Verifying an NNDH (Definition~\ref{defn:nn-hyperprop})
    reduces to verifying an input-output property of the computational graph in this figure.
    The boxes enclose sub-graphs of the computational graph.
    The contents of each box are placeholders.
    Pink nodes~\includegraphics{input-color-legend.pdf} represent inputs, 
    yellow nodes~\includegraphics{param-color-legend.pdf} represent parameters,
    and blue nodes~\includegraphics{output-color-legend.pdf} represent outputs.
    The input and output nodes in each sub-graph are repetitions of 
    their direct predecessors or direct successors outside
    of the subgraph.
    The inputs of~\(\NNSat\) were rearranged for better legibility.
  }\label{fig:self-composition}
\end{figure}

\section{Related Work}
Using self-composition for verifying specific global specifications
was explored previously~\cite{KatzBarrettDillEtAl2017a,KhedrShoukry2022}.
We use self-composition for verifying a range of global specifications.
% Merging a linear inequality over the output into the terminal
% linear layer of a network~\cite{XuZhangWangEtAl2021} is similar
% to our use of~\(\NNSat\), but only applicable to single linear
% inequalities.
Improved encodings of self-composition~\cite{WangHuangZhu2022}
and approaches from differential verification of 
neural networks~\cite{PaulsenWangWang2020}
are interesting directions for improving the verification of NNDHs.

In 2017, verifying global robustness was found to be infeasible 
using the then-available verifiers~\cite{KatzBarrettDillEtAl2017a}.
Recent approaches to global robustness~\cite{WangHuangZhu2022}
and global fairness specifications~\cite{UrbanChristakisWuestholzEtAl2020} 
have demonstrated that verifying global specifications is feasible today.
The reason behind this could be that practically, neural networks 
appear not to realise their full combinatorial potential~\cite{HaninRolnick2019},
in a way that allows for efficient branch-and-bound 
verification~\cite{UrbanChristakisWuestholzEtAl2020}.

\section{Conclusion}
We present a versatile formalism for expressing global specifications 
while maintaining compatibility with existing verification approaches. 
Evaluating this approach empirically remains future work.
A promising verifier for this approach is~\(\alpha,\!\beta\)-CROWN~\cite{ZhangWangXuEtAl2022}, as it already supports verifying arbitrary computational graphs.
An interesting direction is comparing our generally applicable approach 
with approaches specialised to individual global specifications, such as,
global robustness~\cite{WangHuangZhu2022}, dependency fairness~\cite{UrbanChristakisWuestholzEtAl2020} 
and Lipschitz continuity~\cite{BhowmickDSouzaRaghavan2021}.

%
% ---- Bibliography ----
%
% BibTeX users should specify bibliography style 'splncs04'.
% References will then be sorted and formatted in the correct style.
%
\bibliographystyle{splncs04}
\bibliography{references}
\end{document}